\documentclass[10pt,twocolumn,letterpaper]{article}

\usepackage{cvpr}
\usepackage{times}
\usepackage{epsfig}
\usepackage{graphicx}
\usepackage{amsmath}
\usepackage{amssymb}
\usepackage{authblk}
\usepackage{bm}
\usepackage{url}
\usepackage{multirow}
\usepackage{array}
\usepackage{amsthm}
\usepackage[noline,ruled]{algorithm2e}
\newtheorem{mydef}{Theorem}
\usepackage{subfigure}
\usepackage{epstopdf}
\usepackage{mathtools}

\cvprfinalcopy 

\def\cvprPaperID{****} 

\begin{document}

\title{Supervised Hashing based on Energy Minimization}
\author{Zihao Hu}
\author{Xiyi Luo}
\author{Hongtao Lu\thanks{Corresponding author.}}
\author{Yong Yu}
\affil{Department of Computer Science and Engineering, Shanghai Jiao Tong University\authorcr{\{zihaohu, moberq.luo, htlu\}@sjtu.edu.cn, yyu@apex.sjtu.edu.cn}}
\maketitle

\begin{abstract}
   Recently, supervised hashing methods have attracted much attention since they can optimize retrieval speed and storage cost while preserving semantic information. Because hashing codes learning is NP-hard, many methods resort to some form of relaxation technique. But the performance of these methods can easily deteriorate due to the relaxation. Luckily, many supervised hashing formulations can be viewed as energy functions, hence solving hashing codes is equivalent to learning marginals in the corresponding conditional random field (CRF). By minimizing the KL divergence between a fully factorized distribution and the Gibbs distribution of this CRF, a set of consistency equations can be obtained, but updating them in parallel may not yield a local optimum since the variational lower bound is not guaranteed to increase. In this paper, we use a linear approximation of the sigmoid function to convert these consistency equations to linear systems, which have a closed-form solution. By applying this novel technique to two classical hashing formulations KSH and SPLH, we obtain two new methods called EM (energy minimizing based)-KSH and EM-SPLH.
   Experimental results on three datasets show the superiority of our methods.
\end{abstract}

\section{Introduction}

Nearest neighbor search (NNS) is a well-known problem arising in numerous fields of application for finding points closest to a given query. Exact nearest neighbor search is intractable in high-dimensional spaces and unnecessary in many cases. Hence, hashing has been merited for it can refine retrieval speed and storage cost considerably ~\cite{weiss2009spectral,wang2010semi,norouzi2011minimal,liu2012supervised,kong2012isotropic,wang2013learning,lin2013general,yu2014circulant,lin2014fast,shen2015supervised,kang2016column}.

Studies on hashing are roughly in two streams, categorized by whether the learned hash functions rely on the training data. The early exploration of hashing focuses on using random projections to construct hash functions, thus is data-independent. The most popular data-independent method is Locality Sensitive Hashing (LSH)~\cite{gionis1999similarity}, which is widely used until now.

Data-dependent hashing methods have attracted considerable attention in recent years, since they leverage the training data to achieve better performance. Data-dependent hashing can be divided into two types, i.e., unsupervised and supervised methods. Unsupervised techniques learn underlying linear or non-linear local structures of the training data. Representative methods in this fashion include iterative quantization (ITQ) ~\cite{gong2011iterative} and locally linear hashing (LLH) ~\cite{irie2014locally}. However, in many real-world scenarios, preserving semantic information is more important. Hence, supervised hashing methods are proposed to leverage semantic tags of data points. Representatives of supervised hashing methods include CCA-ITQ ~\cite{gong2011iterative}, kernelized supervised hashing (KSH) ~\cite{liu2012supervised}, two step hashing (TSH) ~\cite{lin2013general}, latent factor hashing (LFH) ~\cite{zhang2014supervised}, supervised discrete hashing (SDH) ~\cite{shen2015supervised} and COSDISH ~\cite{kang2016column}.


Although many efforts have been made on designing new formulations and optimization procedures for supervised hashing, there are few works that adopt probabilistic inference techniques. LFH ~\cite{zhang2014supervised} might be the first method to use a generative model for supervised hashing, which assumes that the pairwise similarity is generated by the inner product of two corresponding binary codes. By introducing a prior on hashing codes and applying some form of relaxation, the resulting model is easy to optimize and can yield a satisfactory performance. Bayesian supervised hashing (BSH) ~\cite{hu2017bayesian} adopts the mean-field approach to infer latent factors and tune hyper-parameters automatically. However, BSH models each hashing code with a multivariate Gaussian distribution. When learning $d$-bit codes for $n$ data points, the space complexity is $\mathcal{O}(nd^2)$, which is unbearable for real applications.

In this paper, we propose a novel technique which is also based on probabilistic inference. But instead of modeling the hashing code as a multivariate Gaussian random vector, we view each binary bit as a discrete random variable that only takes values of $1$ and $-1$. Therefore, finding the most probable hashing codes is equivalent to solving marginals in the corresponding CRF. We adopt a fully
factorized mean-field inference to obtain consistency equations, and approximate these 
equations by linear systems instead of iterating them, therefore, a closed-form solution can be obtained. We use this technique to improve kernel-based supervised hashing (KSH) ~\cite{liu2012supervised} and sequential projection learning for hashing (SPLH) ~\cite{wang2010sequential}, and we obtain two new hashing methods called EM (energy minimizing based)-KSH and EM-SPLH. Contributions of this paper are summarized as follows:
\begin{itemize}
	%
	\item We first view solving different supervised hashing formulations as learning marginals in corresponding CRFs, and provide a simple yet effective linear approximation to solve them. Since many supervised hashing loss formulations can be viewed as energy functions, we can consider our method as a general framework that is able to incorporate many of them.
	\item We propose a linear-time variant of EM-KSH to tackle large-scale problems. Besides, we reduce the space complexity from the original $\mathcal{O}(nd^2)$ in BSH to $\mathcal{O}(nd)$.
	\item We have conducted image retrieval experiments on three real-world datasets. Results show that EM-KSH and EM-SPLH outperform state-of-the-art methods. Using the linear time variant of EM-KSH, we can train 64-bit hashing codes on NUS-WIDE in 20 seconds while retaining a state-of-the-art performance. 
\end{itemize}

\section{Notations and Problem Definition}
\subsection{Notations}
Lowercase and uppercase boldface letters
denote vectors and matrices, respectively. For a matrix $\mathbf{A}\in\mathbb{R}^{m\times n}$ , $A_{ij}$ represents the element at the $i$-th row and $j$-th column in $\mathbf{A}$, and $\mathbf{A}^T$ denotes the transpose of $\mathbf{A}$. $\mathbf{A}_{i\cdot}$ and $\mathbf{A}_{\cdot j}$ denote column vectors formed by the $i$-th row and the $j$-th column of $\mathbf{A}$, respectively. When $\mathbf{A}$ is a square matrix, we let $\mathbf{A}^{-1}$ be the inverse  (if exists) of $\mathbf{A}$, and diag$(\mathbf{A})$ be a column vector formed by diagonal elements in $\mathbf{A}$. $\mathbf{I}$ denotes the identity matrix of appropriate size, and $\mathbf{1}$ is a vector or matrix with all ones of appropriate size. $\lVert{\cdot}\rVert_2$ denotes the spectral norm of a matrix while $\lVert{\cdot}\rVert_F$ denotes the Frobenius norm. For a random variable $b$, $\mathbb{E}[b]$ returns its expectation.

\subsection{Problem Definition}
Suppose $\mathbf{X}=[\mathbf{X}_{1\cdot},\cdots,\mathbf{X}_{n\cdot}]^T\in\mathbb{R}^{n\times p}$ is the data matrix, where $\mathbf{X}_{i\cdot}$ is the feature of the $i$-th data point. In conventional settings ~\cite{zhang2014supervised,kang2016column}, the semantic information is given by a pairwise similarity matrix $\mathbf{S}\in\{-1,0,1\}^{n\times n}$, where $S_{ij}=1$ denotes that the $i$-th data point is similar to the $j$-th, while $S_{ij}=-1$ means they are dissimilar. When $S_{ij}=0$, we do not know whether they are similar or not. For ease of illustration, we assume that $\mathbf{S}$ is fully observed, that is, $\mathbf{S}\in\{-1,1\}^{n\times n}$, while our method can tackle the case that partial information is missing naturally. We denote $\mathbf{B}=[\mathbf{B}_{1\cdot},\mathbf{B}_{2\cdot},\cdots,\mathbf{B}_{n\cdot}]^T=[B_{ik}]\in\{-1,1\}^{n\times d}$ as the learned hashing matrix, with $\mathbf{B}_{i\cdot}$ as the $d$-bit hashing code for the $i$-th data point. The purpose of supervised hashing is to preserve semantic similarities in Hamming space, that is, the Hamming distance between hashing codes of similar data points should be small.

\section{Energy Minimizing based Hashing}
\subsection{General Energy Minimizing Approach}
Many supervised hashing loss formulations, for instance, BRE ~\cite{kulis2009learning}, SPLH ~\cite{wang2010sequential}, KSH ~\cite{liu2012supervised} and LFH ~\cite{zhang2014supervised} can  be viewed as  exponential losses of polynomials. Inspired by ~\cite{krahenbuhl2011efficient,krahenbuhl2013parameter}, we view a supervised hashing learning formulation as the corresponding densely connected CRF. If we denote a supervised hashing loss as $\mathcal{E}(\mathbf{B};\mathbf{S})$, the corresponding Gibbs distribution can be written as:

\begin{equation}\label{gibbs}
p(\mathbf{B}|\mathbf{S})=\frac{1}{Z}\exp\{-\mathcal{E}(\mathbf{B};\mathbf{S})\}.
\end{equation}

To obtain marginal probabilities of this distribution, we must calculate
sums over exponentials of energy functions, which is nevertheless intractable. Hence, some form of approximation must be utilized.
The mean-field approximation optimizes a distribution $q(\mathbf{B})$ that minimizes the KL divergence KL$(q||p)$ and factorizes with respect to a partition of variables in 
$\mathbf{B}$. Intuitively, the matrix $\mathbf{B}$ has three ways to factorize: by element, by row, and by column. Figure \ref{fig1} demonstrates a toy graphical model and its three decomposition structures.



\begin{figure}[t]
	\begin{center}
		\includegraphics[width=0.47\textwidth]{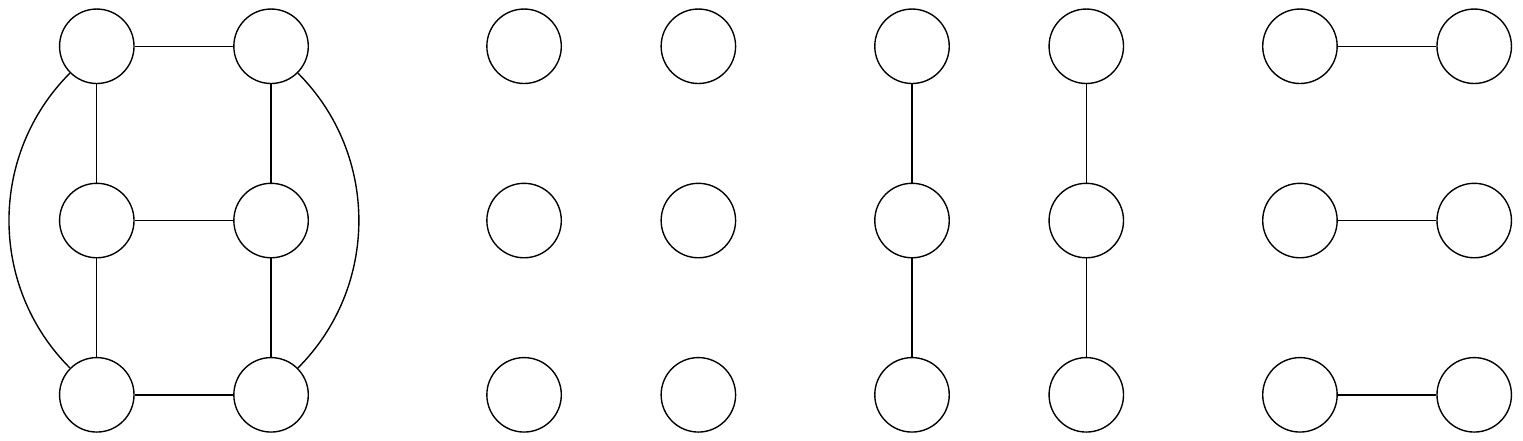}
	\end{center}
	\caption{The graphical model of a toy example and its different decomposition structures. From the left to the right are the original structure, the decomposition by element, by column and by row of the original structure, respectively.}
	\label{fig1}
\end{figure}

Factorizing $\mathbf{B}$ into independent variables provides a tractable way to infer $\mathbf{B}$. Since each element is a binary random variable, a set of closed-form updating equations can be obtained. The drawback of this approach is that all interaction terms between variables are neglected, so the performance can easily deteriorate.

Factorizing $\mathbf{B}$ by row or by column is more challenge since plenty of interaction terms have to be considered during the optimization process.
But if we can obtain the joint distribution of variables by row or by column of $\mathbf{B}$, then a two-round message passing process could yield more precise marginals than the fully factorized mean-field inference.

Our motivation is to maintain the tractability of the fully factorized distribution while considering interactions between rows or columns of variables in $\mathbf{B}$. To achieve this goal, we first derive consistency equations from the fully factorized mean-field inference, then approximate a fixed point of these equations. While solving one row or one column of variables, we view others as constants. Solving these equations is intractable since they contain the sigmoid function $\sigma(x)=1/\big(1+\exp(-x)\big)$. Therefore, we compute a linear approximation of the sigmoid function on a restricted interval and convert the original consistency equations to a set of linear systems.


To see how the sigmoid function is involved, let us consider minimizing the KL divergence between the fully factorized distribution
$q(\mathbf{B}|\mathbf{\Phi})=\prod_{i}^n\prod_{k}^d{\phi}_{ik}^{(B_{ik}+1)/2}{(1-\phi_{ik})}^{(1-B_{ik})/2}$
and $p(\mathbf{B}|\mathbf{S})$ in (\ref{gibbs}), where $\phi_{ik}$ is the probability that $B_{ik}$ takes the value $1$, and $1-\phi_{ik}$ is the probability that $B_{ik}$ takes the value $-1$.
\begin{equation}
\begin{split}
\text{KL}(q||p)&=\sum_{i}^n\sum_{k}^d\phi_{ik}\ln \phi_{ik}+(1-\phi_{ik})\ln(1-\phi_{ik})\\
&+\mathbb{E}_q[\mathcal{E}]+\ln Z.
\end{split}
\end{equation}
After letting the derivative of $\phi_{ik}$ be zero, we obtain
\begin{equation}\label{sigmoid}
\phi_{ik}=\sigma\Big(-\frac{\partial \mathbb{E}_q[\mathcal{E}]}{\partial \phi_{ik}}\Big).
\end{equation}
Since $\mathcal{E}$ is a polynomial of $\mathbf{B}$ and for a binary random variable $b$ which takes the value $1$ with probability $\phi$ and the value $-1$ with probability $1-\phi$, $\mathbb{E}[b]=1\times\phi+(-1)\times(1-\phi)=2\phi-1$, ~(\ref{sigmoid}) usually has the form of
\begin{equation}\label{hashing}
\bm{\phi}=\sigma(\mathbf{A}(2\bm{\phi}-\mathbf{1})+\mathbf{b}),
\end{equation}
%
where $\mathbf{A}$ is a real symmetric matrix of $d\times d$ and $\mathbf{b}$ is a real vector of $d\times 1$. We wish to solve $\bm{\phi}\in[0,1]^{d\times 1}$ approximately.

We restrict the range in which we approximate the sigmoid function to make the approximation error small by defining
\begin{equation}
\lambda=\max\limits_{i\in\{1,\cdots,d\}}\Big(\sum_j^d|A_{ij}|+|b_i|\Big)/c,
\end{equation}
where $c>0$ is a constant, and solve the scaled problem
\begin{equation}\label{equation}
\bm{\phi}=\sigma(\lambda^{-1}(\mathbf{A}(2\bm\phi-\mathbf{1})+\mathbf{b})).
\end{equation}
Now each term inside the sigmoid function is  bounded by an interval $[-c,c]$. We compute a linear approximation of the sigmoid function: $\sigma(x)\approx c_1x+c_2$ on this interval, where $c_1$ and $c_2$ can be determined by minimizing
\begin{equation}\label{c1}
\min_{c_1,c_2}
\int_{-c}^{c}(\sigma(x)-c_1x-c_2)^2dx.
\end{equation}
Each element in the sigmoid function of  (\ref{equation}) is a linear combination of $\bm{\phi}$, so approximating these elements directly may cause considerable error. When $\mathbf{A}$ is invertible, we apply a linear transformation $\mathbf{v}=\lambda^{-1}\mathbf{A}(2\bm{\phi}-\mathbf{1})$, that is, $2\bm{\phi}-\mathbf{1}=\lambda\mathbf{A}^{-1}\mathbf{v}$, to make sure that each term in the sigmoid function involves one variable in $\mathbf{v}$. After this approximation, (\ref{equation}) turns out to be linear equations of $\mathbf{v}$ as
\begin{equation}\label{invert}
(\lambda\mathbf{A}^{-1}-2c_1\mathbf{I})\mathbf{v}=2c_1\lambda^{-1}\mathbf{b},
\end{equation}
where since $c_2\equiv 0.5$, it is eliminated automatically. 

We mainly discuss the case that $\mathbf{A}$ is invertible, while the case that $\mathbf{A}$ is singular can be treated similarly. In both cases, we have to inverse $(\lambda\mathbf{I}-2c_1\mathbf{A})$, so we use the following theorem to ensure its invertibility.


\begin{mydef}\label{theorem}
	A sufficient condition for the invertibility of $(\lambda\mathbf{I}-2c_1\mathbf{A})$ is  that $2c_1<1/c$.
\end{mydef}
\begin{proof}
	We denote the eigenvalue of $\mathbf{A}$ with the largest magnitude as $\lambda_m$. According to the Gershgorin circle theorem and $\lambda=\max\limits_{i\in\{1,\cdots,d\}}\Big(\sum_j^d|A_{ij}|+|b_i|\Big)/c$, we have
	\begin{equation}
	|\lambda_m|\leq \max\limits_{i\in\{1,\cdots,d\}}\Big(\sum_j^d|A_{ij}|\Big)\leq \lambda c.
	\end{equation}
	If $2c_1<1/c$, the eigenvalue of $2c_1\mathbf{A}$ with the largest magnitude is less than $\lambda$, so $(\lambda\mathbf{I}-2c_1\mathbf{A})$ would be positive definite, which concludes the proof.
\end{proof}	

By computing, we find that when $c<2.5997$, the condition in Theorem \ref{theorem} holds automatically.


According to whether $\mathbf{b=0}$, (\ref{invert}) has two cases:
\begin{itemize}
	\item For $\mathbf{b}\neq\mathbf{0}$, this problem has a closed-form solution
	\begin{equation}\label{solution}
	\mathbf{v}=2c_1\lambda^{-1}(\lambda\mathbf{A}^{-1}-2c_1\mathbf{I})^{-1}\mathbf{b}.
	\end{equation}
	\item For $\mathbf{b}=\mathbf{0}$, this linear system does not have non-zero solution, but we can find the solution of
	\begin{equation}
	\min_{\mathbf{v}}
	\left\lVert {(\lambda\mathbf{A}^{-1}-2c_1\mathbf{I})}\mathbf{v}\right\rVert_2^2
	\end{equation}
	instead. The solution of this problem is the eigenvector associated with the smallest eigenvalue of the matrix $(\lambda\mathbf{A}^{-1}-2c_1\mathbf{I})^T(\lambda\mathbf{A}^{-1}-2c_1\mathbf{I})$.
\end{itemize}
After solving $\mathbf{v}$, recall that $\mathbf{v}+\lambda^{-1}\mathbf{b}
\triangleq\mathbf{v}^{\prime}\in[-c,c]$, we first re-normalize $\mathbf{v}^{\prime}$ by 

\begin{equation}
\mathbf{v}^\prime\coloneqq c\bigg(2\Big(\frac{\mathbf{v}^\prime-\text{min}(\mathbf{v}^\prime)}{\text{max}(\mathbf{v}^\prime)-\text{min}(\mathbf{v}^\prime)}\Big)-1\bigg),
\end{equation}
then use $\bm{\phi}=\sigma(\mathbf{v}^\prime)$ to obtain the final $\bm{\phi}$.

\subsection{Applications to Supervised Hashing}

We use two supervised hashing formulations to illustrate how to derive our EM-KSH and EM-SPLH.

For KSH ~\cite{liu2012supervised}, the corresponding Gibbs function is:
\begin{eqnarray}
p(\mathbf{B}|\mathbf{S})=\exp{\Big\{-\frac{1}{4}\sum_{i<j}^n(\mathbf{B}_{i\cdot}^T\mathbf{B}_{j\cdot}-dS_{ij})^2\Big\}}.
\end{eqnarray} 
By minimizing the KL divergence between $q(\mathbf{B}|\mathbf{\Phi})$ and $p(\mathbf{B}|\mathbf{S})$, the optimal solution is given by:
\begin{equation}\label{ksh}
\begin{split}
\phi_{ik}&=\sigma\Big(-\sum_{j\neq i}^n\sum_{k^{\prime}\neq k}^d(2\phi_{jk}-1)(2\phi_{jk^{\prime}}-1)(2\phi_{ik^{\prime}}-1)\\
&+\sum_{j\neq i}^n dS_{ij}(2\phi_{jk}-1)\Big).
\end{split}
\end{equation}
Letting 
\begin{equation}\label{phi1}
\begin{split}
&A_{kk^{\prime}}^i=\sum_{j\neq i}^n-1_{[k\neq k^\prime]}(2\phi_{jk}-1)(2\phi_{jk^{\prime}}-1),\\
&\ \ b_{k}^i\ \ =\sum_{j\neq i}^n dS_{ij}(2\phi_{jk}-1),\ \ i\in\{1,\cdots,n\},
\end{split}
\end{equation}
where $1_{[\cdot]}$ is the indicator function, we recognize that (\ref{ksh}) are actually simultaneous equations with the form of (\ref{hashing}) over triplets $\{\mathbf{A}^i,\mathbf{b}^i,\bm{\Phi}_{i\cdot}\}$ where $i\in\{1,\cdots,n\}$, so they can be solved using the technique mentioned above. 


Another example is SPLH ~\cite{wang2010sequential}. Its Gibbs distribution can be expressed as:
\begin{equation}
p(\mathbf{B}|\mathbf{S})=\exp\Big\{{\frac{1}{2}\sum_{i<j}^n S_{ij}\mathbf{B}_{i\cdot}^T\mathbf{B}_{j\cdot}}\Big\}.
\end{equation}
The corresponding re-estimation equations are given by:
\begin{equation}
\phi_{ik}=\sigma\Big(\sum_{j\neq i}^nS_{ij}(2\phi_{jk}-1)\Big).
\end{equation}

Since $\mathbf{A}^{k}=\mathbf{S}$ and $\mathbf{b}^{k}=\mathbf{0}$ hold for all $k\in\{1,\cdots,d\}$, these equations can be decoupled into $d$ identical problems over triplets $\{\mathbf{S},\mathbf{0},\mathbf{\Phi}_{\cdot k}\}$. Therefore, all hashing bits will be the same 
after solving these linear systems. In addition, $\mathbf{\Phi}_{\cdot k}$ is independent of $\mathbf{S}$, so we can solve a fixed point in one iteration exactly.


It is worth noting that after approximating consistency equations of KSH, the parameter $\mathbf{b}\neq \mathbf{0}$, while for SPLH, $\mathbf{b}\equiv\mathbf{0}$. Consequently, they correspond to two cases of our linear approximation method, respectively. We shall evaluate both of them in the experimental section. 

\subsection{Stochastic Learning}
Due to the unbearable time cost and space complexity for tackling the whole similarity matrix, we 
follow the same method as LFH~\cite{zhang2014supervised} to reduce the complexity, i.e., sampling $m$ columns in the original similarity matrix randomly. The resulting sub-matrix is denoted as $\mathbf{S}\in\{-1,1\}^{n\times m}$. 
We partition the learned $\mathbf{\Phi}$ as $\mathbf{\Phi}=[\mathbf{\Phi}_1,\mathbf{\Phi}_2]^T$ where $\mathbf{\Phi}_1\in [0,1]^{m\times m}$ is the hashing matrix for anterior $m$ points, while $\mathbf{\Phi}_2\in [0,1]^{(n-m)\times m}$ is the hashing matrix for the later $(n-m)$ data points.
We show how to learn semantic information in $\mathbf{S}$ as follows.
For $i\in\{1,\cdots,m\}$, $\mathbf{A}^i$ and $\mathbf{b}^i$ are exactly the same as in (\ref{phi1}). For $i\in\{m+1,\cdots,n\}$, we have
\begin{equation}\label{phi2}
\begin{split}
&A_{kk^{\prime}}^i=\sum_{j=1}^m-1_{[k\neq k^\prime]}(2\phi_{jk}-1)(2\phi_{jk^{\prime}}-1),\\
&\ \ b_{k}^i\ \ =\sum_{j=1}^m dS_{ij}(2\phi_{ik}-1),\ \ i\in\{m+1,\cdots,n\}.
\end{split}
\end{equation}
It is interesting to find out that for all $i\in\{m+1,\cdots,n\}$, $\mathbf{A}^{i}$ is the same matrix as $\mathbf{A}$. This property can be used to reduce the time and space complexity of EM-KSH. To see this, we first write down corresponding linear systems:
\begin{equation}\label{equation2}
(\lambda_i\mathbf{A}^{-1}-2c_1\mathbf{I})\mathbf{v}_i=2c_1\lambda_i^{-1}\mathbf{b}_i,\ i\in\{m+1,\cdots,n\}.
\end{equation}
We perform the eigen decomposition of $\mathbf{A}$ as $\mathbf{A}=\mathbf{PDP^{-1}}$ to accelerate the calculation, and solve $\mathbf{v}_i$ as
\begin{eqnarray}\label{trick}
\begin{aligned}
\mathbf{v}_i&=2c_1\mathbf{P}\big(\lambda_i\mathbf{D}^{-1}-2c_1\mathbf{I} \big)^{-1}\mathbf{P}^{-1}\frac{\mathbf{b}^i}{\lambda_i} \\
&=2c_1\mathbf{P}\text{diag}\Big ((\lambda_i\mathbf{D}^{-1}-2c_1\mathbf{I} )^{-1}\Big)\odot(\mathbf{P}^{-1}\frac{\mathbf{b}^i}{\lambda_i}),
\end{aligned}
\end{eqnarray}
where $\odot$ is the Hadamard product (element-wise product) of two matrices (vectors). Hence, we do not need to spend $\mathcal{O}(d^2(n-m))$ space to store all $(\lambda_i\mathbf{D}^{-1}-2c_1)$, instead we could complete necessary calculations at the cost of $\mathcal{O}(d(n-m)+d^2)$. Besides, $\mathbf{\Phi}_{i\cdot}$ is independent of $\mathbf{A}$ for $i\in\{m+1,\cdots,n\}$, so we can solve $\mathbf{\Phi}_2$ in one iteration. Our learning process is summarized in Algorithm \ref{Learning_Procedure_for_EM-KSH}. In general, $2\leq T\leq 10$ is enough to get satisfactory performance. We set $T=3$ throughout our experiments for time-saving purposes.
\begin{algorithm}
	\SetKwInOut{Input}{Input}
	\SetKwInOut{Output}{Output}
	\caption{Learning procedure for EM-KSH}
	\label{Learning_Procedure_for_EM-KSH}
	\Input{$\mathbf{S}\in\{-1,1\}^{n\times m},c,d, T.$}
	\Output{$\mathbf{\Phi}\in [0,1]^{n\times d}$, which will be rounded to obtain a hashing matrix $\mathbf{B}\in\{-1,1\}^{n\times d}$.}
	\BlankLine
	Initialize the matrix $\mathbf{\Phi}$ by randomization.\\
	\For{$t\leftarrow 1$ to $T$}{
		Construct $\mathbf{A}^i,\mathbf{b}^i$ according to (\ref{phi1}),\\ where $i\in\{1,\cdots,m\}$.\\ 
		Solve the resulting $m$ linear systems using (\ref{solution}),\\
		and compute $\mathbf{\Phi}_1$.\\  
	}
	Construct $\mathbf{A},\mathbf{b}^i$ according to (\ref{phi2}), \\ where $i\in\{m+1,\cdots,n\}$.\\ 
	Compute the eigendecomposition of $\mathbf{A}$.\\
	Solve the resulting $n-m$ linear systems using (\ref{trick}), \\
	and compute $\mathbf{\Phi}_2$.\\
	\Return{$\mathbf{\Phi}=[\mathbf{\Phi}_1,\mathbf{\Phi}_2]^T$.}
\end{algorithm}

\subsection{Rounding and Out-of-Sample Extension }
We follow the same procedure as in BSH~\cite{hu2017bayesian} for both rounding and out-of-sample extension. That is, we round $\mathbf{\Phi}$ according to the mean value of each bit. For out-of-sample extension, we simply learn a linear mapping $\mathbf{W}$ from $\mathbf{X}$ to $\mathbf{\Phi}$ by minimizing
\begin{equation}
\min_{\mathbf{W}}{\lVert\mathbf{\Phi}-\mathbf{XW}\rVert}_F^2+\lambda_h\lVert\mathbf{W}\rVert_F^2,
\end{equation}
where $\lambda_h$ is a regularization hyper-parameter. For a new data point $\mathbf{x}$, the corresponding $\bm{\phi}$ is calculated as
\begin{equation}
\bm{\phi}=\mathbf{W}^T\mathbf{x}.
\end{equation}
\subsection{Complexity Analysis}
We discuss the case that $\mathbf{S}\in\{-1,1\}^{n\times m}$ is a sub-matrix of the original similarity matrix. 
For $i\in\{1,\cdots,m\}$, we need $\mathcal{O}(m(n-1)d^2+m(n-1)d)=\mathcal{O}(mnd^2)$ time to calculate all $\mathbf{A}^i$ and $\mathbf{b}^i$. Then it takes $\mathcal{O}(d^3)$ to solve an equation as in (\ref{equation}), hence, calculating $\mathbf{\Phi}_1$ costs $\mathcal{O}(mnd^2+md^3)$. For $i\in\{m+1,\cdots,n\}$, $\mathcal{O}(md^2+(n-m)md)$ is required to obtain the common $\mathbf{A}$ and all $\mathbf{b}^i$. Then, by using the trick in (\ref{trick}), we are able to solve one problem like this in $\mathcal{O}(d^2)$ time, and $\mathcal{O}(d^3)$ is needed to compute the eigendecomposition of $\mathbf{A}$. 
Since $d\ll m$, updating $\mathbf{\Phi}_2$ needs $\mathcal{O}((n-m)md)$. Provided that we have to compute $\mathbf{\Phi}_1$ for $T$ times and $\mathbf{\Phi}_2$ in one iteration, the total time of computing $\mathbf{\Phi}$ is bounded by $\mathcal{O}(Tnmd^2)$. Since $m$ is usually chosen as a constant like $1000$ and the factor in $\mathcal{O}(\cdot)$ is usually quite small, this method is very fast in real applications.





For the space complexity, $\mathcal{O}({m(d^2+d)})$
is occupied by $\mathbf{A}^i$ and $\mathbf{b}^i$ for $i\in\{1,\cdots,m\}$.
By utilizing the trick in (\ref{trick}), it only takes $\mathcal{O}((n-m)d+d^2)$ space for learning $\mathbf{\Phi}_2$. Therefore, the total storage cost  is $\mathcal{O}(md^2+nd)$. In most cases, we have $md=\mathcal{O}(n)$, hence the space complexity can be written as $\mathcal{O}(nd)$, which outperforms the original $\mathcal{O}(nd^2)$ in BSH~\cite{hu2017bayesian}. 
\subsection{Extensions of Other Methods}
Extending our  method to BRE ~\cite{kulis2009learning} and ExpH ~\cite{lin2013general} is quite straightforward. Here we only show how to extend our proposed technique to LFH ~\cite{zhang2014supervised}. The optimization objective of LFH is
\begin{eqnarray}
p(\mathbf{B}|\mathbf{S})=\prod_{i<j}^n\sigma(\mathbf{B}_{i\cdot}^T\mathbf{B}_{j\cdot})^{\frac{1+S_{ij}}{2}}\big(1-\sigma(\mathbf{B}_{i\cdot}^T\mathbf{B}_{j\cdot})\big)^{\frac{1-S_{ij}}{2}},
\end{eqnarray} 
which is not a Gibbs distribution at first glance. But after applying the local variational method in ~\cite{jaakkola2000bayesian}, the original objective is lower bounded by
\begin{equation}
\tilde{p}(\mathbf{B}|\mathbf{S})=\frac{1}{Z}\exp\Big\{\sum_{i<j}^n\big(\frac{1}{2}S_{ij}\mathbf{B}_{i\cdot}^T\mathbf{B}_{j\cdot}+\lambda(\xi_{ij})(\mathbf{B}_{i\cdot}^T\mathbf{B}_{j\cdot})^2\big)\Big\},
\end{equation}
where
$\xi_{ij}=\sqrt{\mathbb{E}[\mathbf{B}_{i\cdot}^T\mathbf{B}_{j\cdot}]^2}$,
and
$\lambda(\xi_{ij})=-\frac{\sigma(\xi_{ij})-\frac{1}{2}}{2\xi_{ij}}$.
The resulting mean-field consistency equations are given by 
\begin{equation}\label{lfh}
\begin{split}
\phi_{ik}&=\sigma\Big(\sum_{j\neq i}^n\sum_{k^{\prime}\neq k}^d\big(4\lambda(\xi_{ij})(2\phi_{jk}-1)(2\phi_{jk^{\prime}}-1) \\
&(2\phi_{ik^{\prime}}-1)\big)
+\sum_{j\neq i}^n S_{ij}(2\phi_{jk}-1)\Big).
\end{split}
\end{equation}
So we can solve these consistency equations by finding a fixed point in the same manner as EM-KSH.

Interestingly, we have found a deep connection between KSH and LFH. Suppose that for arbitrary two codes $\mathbf{B}_{i\cdot}$ and $\mathbf{B}_{j\cdot}$, we have $\mathbb{E}[\mathbf{B}_{i\cdot}]$ = $\mathbb{E}[\mathbf{B}_{j\cdot}]$ or $\mathbb{E}[\mathbf{B}_{i\cdot}]$ = $-\mathbb{E}[\mathbf{B}_{j\cdot}]$, then $\xi_{ij}=d$ and $\lambda(\xi_{ij})\approx-\frac{1}{4d}$. Substituting $\lambda(\xi_{ij})\approx-\frac{1}{4d}$ into (\ref{lfh}) yields
\begin{equation}
\begin{split}
\phi_{ik}&=\sigma\Big(-\sum_{j\neq i}^n\sum_{k^{\prime}\neq k}^d\frac{1}{d}\big((2\phi_{jk}-1)(2\phi_{jk^{\prime}}-1) \\
&(2\phi_{ik^{\prime}}-1)\big)
+\sum_{j\neq i}^n S_{ij}(2\phi_{jk}-1)\Big),
\end{split}
\end{equation}
which is identical to (\ref{ksh})
except for a factor of $d$. Consequently, KSH can be viewed as a hard assignment of hashing codes to $-1$ or $1$, while LFH makes a soft assignment based on probabilities.
\section{Experiments}
\subsection{Datasets}
\begin{table*}
	\begin{center}
		\begin{tabular}{|p{1.6cm}<{\centering}||p{1.2cm}<{\centering}|p{1.2cm}<{\centering}|p{1.2cm}<{\centering}|p{1.2cm}<{\centering}||p{1.2cm}<{\centering}|p{1.2cm}<{\centering}|p{1.2cm}<{\centering}|p{1.2cm}<{\centering}|}
			\hline
			\multicolumn{1}{|c||}{Method}&\multicolumn{4}{c||}{ESPGAME}&\multicolumn{4}{c|}{CIFAR-10} \\
			\hline
			\hline
			
			& 8 bits & 16 bits & 32 bits & 64 bits & 8 bits & 16 bits & 32 bits & 64 bits\\
			\hline
			CCA-ITQ & 0.2751 & 0.2805 &	0.2804 & 0.2816 &
			0.2163 & 0.2215 & 0.0.2254 & 0.2319
			\\
			\hline
			KSH & 0.2977 &0.3086 &0.3194 &0.3248 &0.2521& 0.2825& 0.3210& 0.3492
			\\
			\hline
			LFH & 0.3116&0.3330&0.3546&0.3659&
			0.2881&0.3996&0.5216&0.6085
			\\
			\hline
			SDH & 0.3094&	0.3290&	0.3312&	0.3388
			&
			0.3329&	0.4833&	0.5397&	0.5865
			\\
			\hline
			COSDISH & 0.2977&	0.3158&	0.3327&	0.3407
			&
			\textbf{0.4898}&	\textbf{0.5733}&	\textbf{0.6215}&	\textbf{0.6369}
			\\
			\hline
			NSH & 0.2946&	0.3058&	0.3116&	0.3223
			&
			0.3907&	0.4476&	0.4875&	0.5298
			\\
			\hline
			BSH & 0.3314&	0.3456&	0.3583&	0.3639
			&
			0.4132&	0.4989&	0.5792&	0.6137
			\\
			\hline
			EM-KSH & \textbf{0.3430}&	\textbf{0.3602}&	\textbf{0.3592}&	\textbf{0.3710}
			&
			0.4459&	0.5344&	0.5804&	0.6276
			\\
			\hline
		\end{tabular}
	\end{center}
	\caption{Experimental performance on ESPGAME and CIFAR-10 in terms of mAP. Best results are in bold.}
	\label{table:1}
\end{table*}
\begin{table*}
	\begin{center}
		\begin{tabular}{|p{1.6cm}<{\centering}||p{1.2cm}<{\centering}|p{1.2cm}<{\centering}||p{1.2cm}<{\centering}|p{1.2cm}<{\centering}||p{1.2cm}<{\centering}|p{1.2cm}<{\centering}||p{1.2cm}<{\centering}|p{1.2cm}<{\centering}|}
			\hline
			\multicolumn{1}{|c||}{Code length}&\multicolumn{2}{c||}{8 bits}&\multicolumn{2}{c||}{16 bits}&\multicolumn{2}{c||}{32 bits}&\multicolumn{2}{c|}{64 bits}\\
			\hline
			\hline
			
			& mAP & Time & mAP & Time & mAP & Time & mAP & Time\\
			\hline
			CCA-ITQ & 0.2947&  4.38&	0.3002&	4.65& 0.3111& 6.02&	0.3152&  13.76
			\\
			\hline 
			KSH & 0.3732 &211.44 &0.4149 &732.87 &0.4356& 1792.90& 0.4391& 2931.97
			\\
			\hline
			LFH & 0.4542&	31.05& 0.4854&	54.50& 0.5120& 82.29&	0.5324& 	138.72
			\\
			\hline
			SDH & 0.4353&	46.14& 0.4630&	56.59& 0.4847& 155.17&	0.5143&  649.32
			\\
			\hline
			COSDISH & 0.4299&	18.71& 0.4827& 28.07& 	0.4918& 276.28&	0.5231&  1030.22
			\\
			\hline
			NSH & 0.3837&	17.72& 0.4197&	19.32& 0.4438&	24.51& 0.4385 & 31.95
			\\
			\hline
			BSH & 0.4683&	15.19& 0.4825& 20.85&	0.5076&	39.17& 0.5220&  	121.24
			\\
			\hline
			EM-KSH & \textbf{0.4878}& 10.36&\textbf{0.5120}& 10.53&\textbf{0.5331}& 12.12&\textbf{0.5434}& 17.52
			\\
			\hline
		\end{tabular}
	\end{center}
	\caption{The mAP and the corresponding training time (in seconds) on NUS-WIDE. Best results are in bold.}
	\label{table:2}
\end{table*}
We evaluate our proposed method on three image datasets: NUS-WIDE\footnote{http://lms.comp.nus.edu.sg/research/NUS-WIDE.htm}~\cite{nus-wide-civr09}, CIFAR-10\footnote{http://www.cs.toronto.edu/\url{~}kriz/cifar.html}~\cite{krizhevsky2009learning} and ESPGAME\footnote{{http://www.hunch.net/\url{~}jl/}}. All of them have been widely used for supervised hashing methods evaluation ~\cite{shen2015supervised,kang2016column,hu2017bayesian}.

CIFAR-10 consists of 60,000 color images which are manually categorized into 10 classes. Each image in this dataset is represented by a 512-dimensional GIST feature vector. Two images are considered to be similar if they are of the same class, otherwise, they are treated as dissimilar.

The ESPGAME dataset contains 20,770 images with 268 keywords while the NUS-WIDE dataset includes 269,648 natural images with 81 tags. 
During experiments, we use 512-dimensional GIST features and 500-dimensional bag-of-words features for ESPGAME and NUS-WIDE, respectively. For these two datasets, two images are considered as semantic neighbors if they share at least one common tag.	
%


\subsection{Experimental Settings}\label{exp}

Following ~\cite{zhang2014supervised,kang2016column}, for all datasets, we randomly select 1000 data points as the validation set and 1000 points as the query set. In the preprocessing phase, we perform normalization on features to make each dimension have zero mean and same variance.
The default value of $c$ is $2$, so the linear approximation of the sigmoid function is $\sigma(x)\approx 0.2109x+0.5$ on $[-2,2]$, while the default value of $m$ is $1000$. 
Since our method directly solves a fixed point of consistency mappings, very few iterations are already enough to obtain satisfactory results. For all experiments, we set $T=3$. 

Since supervised hashing methods outperform unsupervised ones in preserving semantic similarities,
we simply compare our proposed method with other state-of-the-art supervised hashing methods, including CCA-ITQ ~\cite{gong2011iterative}, KSH ~\cite{liu2012supervised}, LFH~\cite{zhang2014supervised}, SDH~\cite{shen2015supervised}, COSDISH~\cite{kang2016column}, NSH~\cite{liunatural} and BSH~\cite{hu2017bayesian}. The code of all these methods is implemented by corresponding authors.
For all methods, we follow settings the same as those suggested by these authors. 
All our experiments are conducted on a workstation with 16 Intel i7-6900K CPU cores and 64GB RAM, and all the results are the average value of 10 random partitions.

For these three datasets, we report the compared results in terms of Hamming ranking. For each query, all the data points in the training set are sorted ascending according to the Hamming distance between their hashing codes and the code of the query. The mean average precision (mAP) is used to evaluate different methods.
\subsection{Comparing EM-KSH with baselines}

Table \ref{table:1} and \ref{table:2} show the mAP of our proposed EM-KSH and other methods on these three datasets. In addition, the training time in seconds on the NUS-WIDE dataset with various code lengths is reported in Table \ref{table:2}. By comparing our EM-KSH with other baselines, we discover that for ESPGAME and NUS-WIDE, our method outperforms other baselines in almost all cases. Especially, EM-KSH achieves much better performance than the original KSH for all these three datasets consistently, which justifies that the fixed point solved by EM-KSH is quite desirable. For CIFAR-10, as we can see, COSDISH is the most effective hashing method.
Recall that for a single-label dataset, the equality relation over the set of labels is transitive, that is, for three labels $l_i$, $l_j$, $l_k$, if $l_i=l_j$ and $l_j=l_k$, then $l_i=l_k$. We speculate that the error bound of the $2d$-approximation algorithm being used in COSDISH would tighten when the transitive relation exists in the similarity matrix.


For the time complexity, the performance of our method is even more superior. As we can see, only CCA-ITQ can be slightly faster than  EM-KSH, but our method outperforms CCA-ITQ by a large margin. Besides, EM-KSH is orders of magnitude faster than other state-of-the-art methods, especially for learning 64-bit hashing codes on NUS-WIDE. In fact, the time spent for EM-KSH to learn 64-bit codes on 
NUS-WIDE is less than or equal to the time that other baselines used to learn 8-bit codes. 


In summary, our EM-KSH yields the best performance on two datasets with multiple tags and is almost as fast as CCA-ITQ. Although COSDISH outperforms EM-KSH on the CIFAR-10 dataset, we still argue that our method is state-of-the-art. First, multi-label images are easier to collect and have more real-world applications. Second, EM-KSH is much faster than COSDISH, so we can train longer hashing codes to defeat COSDISH on single-label datasets 
with a reasonable time cost. 
In addition, our method can be easily modified to accommodate different hashing formulations, while COSDISH is less flexible.
\subsection{Sensitivity to Hyper-parameters}
We vary the hyper-parameter $c$ from $0.5$ to $2.5$, and report the mAP performance on CIFAR-10 and NUS-WIDE in Figure \ref{fig2}. We find that EM-KSH is not sensitive to the value of $c$ and can achieve good performance consistently.

In Figure \ref{fig3}, we show a performance comparison of EM-KSH and other baselines with the size of the query set $q$ and the number of columns in the similarity sub-matrix $m$ take values from $1000,2000,3000,5000$ and $10000$ on NUS-WIDE. In almost all cases, the performance of EM-KSH is superior to other baselines.

\begin{figure}[t]
	\begin{center}
		\includegraphics[width=.23\textwidth]{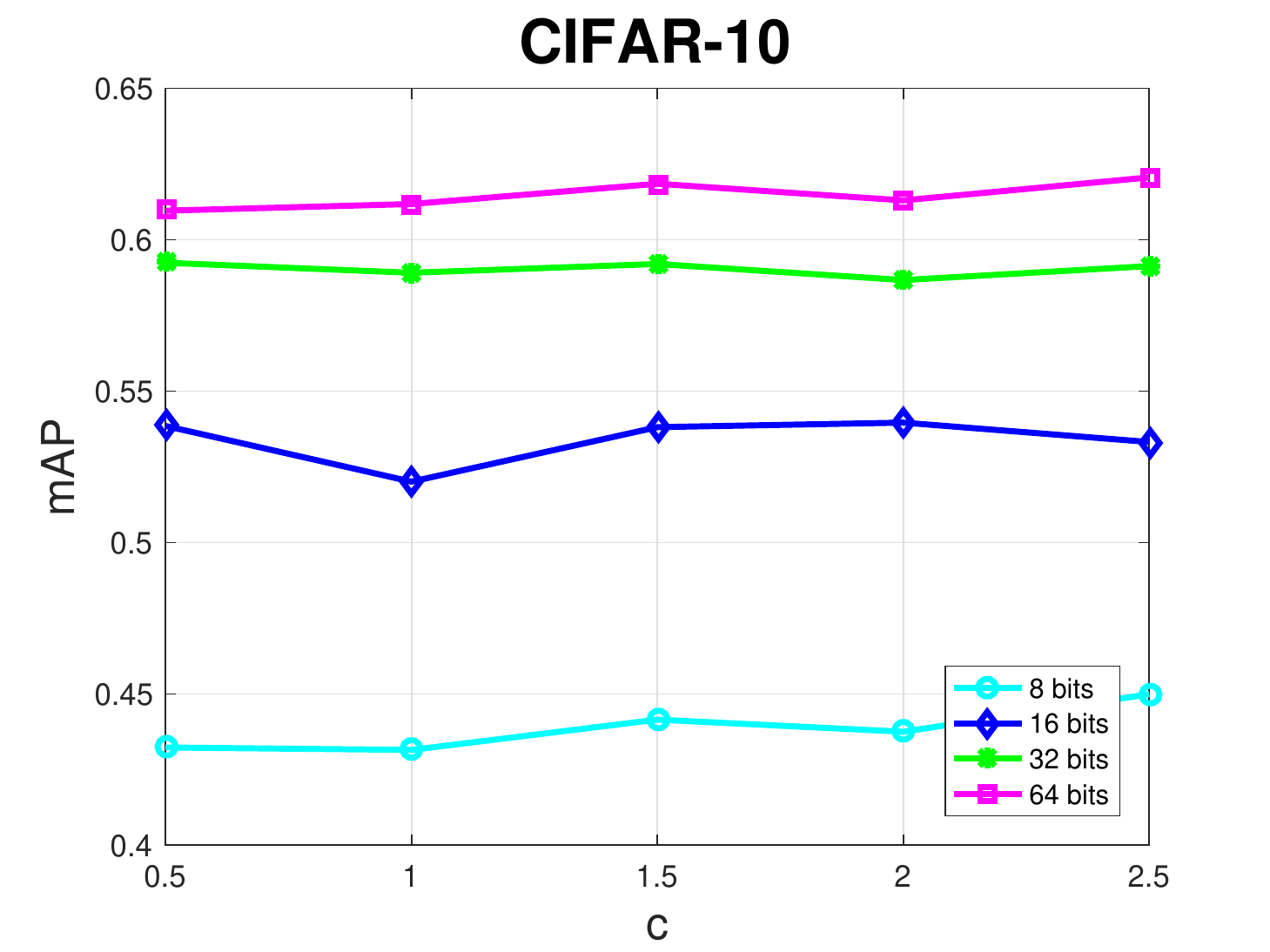}
		\includegraphics[width=.23\textwidth]{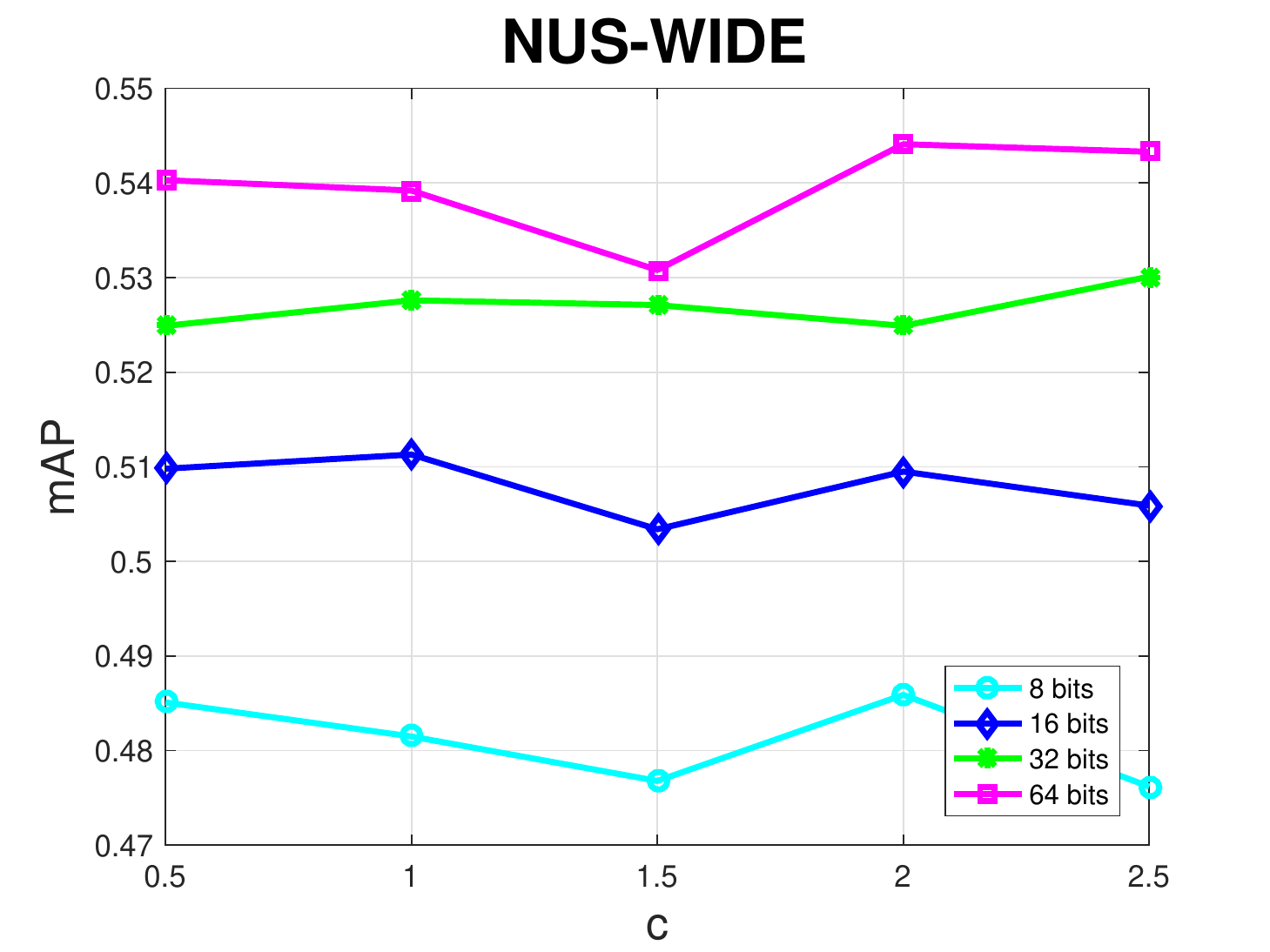}
	\end{center}
	\caption{Sensitivity to hyper-parameter $c$.}
	\label{fig2}
\end{figure}
\begin{figure}[t]
	\begin{center}
		\includegraphics[width=.23\textwidth]{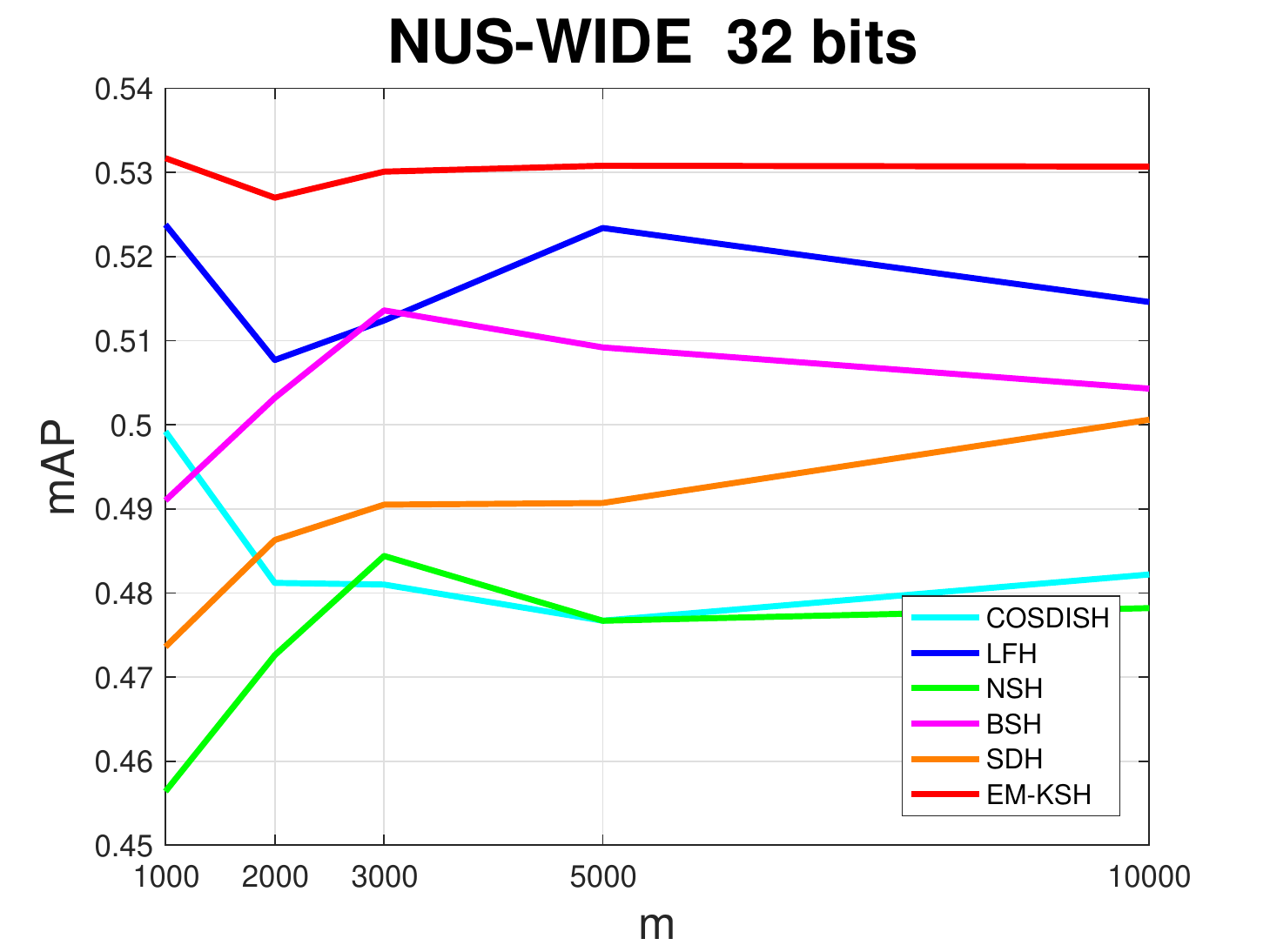}
		\includegraphics[width=.23\textwidth]{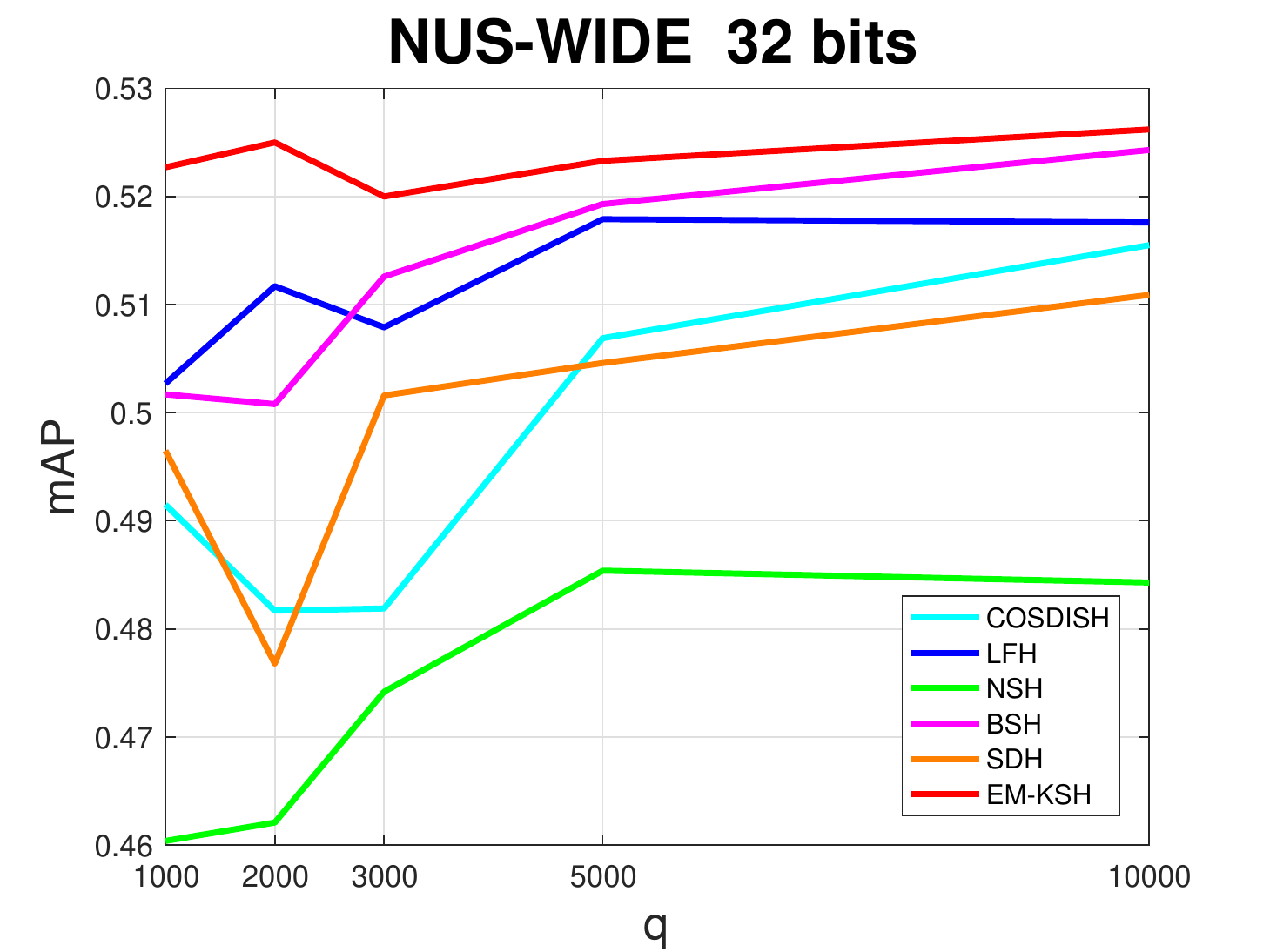}
	\end{center}
	\caption{Sensitivity to hyper-parameters $m$ and $q$.}
	\label{fig3}
\end{figure}
\subsection{Long Hashing Codes Learning}
In many applications, space is not the primary bottleneck, so longer hashing codes might be required to further boost the retrieval precision. We train 128-bit and 256-bit codes using EM-KSH and other baselines on NUS-WIDE and show the result in Table ~\ref{table:4} to further demonstrate the effectiveness of our method.

\begin{table}[!h]
	\begin{center}
		\begin{tabular}{|p{1.6cm}<{\centering}||p{1.0cm}<{\centering}|p{1.1cm}<{\centering}||p{1.0cm}<{\centering}|p{1.2cm}<{\centering}|}
			\hline
			\multicolumn{1}{|c||}{Code length}&\multicolumn{2}{c||}{128 bits}&\multicolumn{2}{c|}{256 bits}\\
			\hline
			& mAP & Time & mAP & Time \\
			\hline
			\hline
			COSDISH & 0.5318&		691.73& 0.5405&	3096.79
			\\
			\hline
			LFH & 0.5378&	71.24&	0.5548& 135.23
			\\
			\hline
			NSH & 0.4415&	22.82& 	0.4606&	43.94
			\\
			\hline
			BSH & 0.5285 &	508.64 & OOM  	&OOM
			\\
			\hline
			EM-KSH & \textbf{0.5509}&	27.12 & 	\textbf{0.5615} &	55.68
			\\
			\hline
		\end{tabular}
	\end{center}
	\caption{The mAP and corresponding training time in seconds on NUS-WIDE for learning 128-bit and 256-bit codes. Best results are in bold. OOM means out-of-memory error.}
	\label{table:4}
\end{table}
During experiments, BSH runs out of memory and terminates while learning 256-bit codes, while our method can easily learn codes up to 256 bits with a rather low time cost. The results further indicates our improvement on the space side is significant.

\subsection{Comparing EM-SPLH with baselines}
EM-KSH corresponds to the case $\mathbf{b}\neq\mathbf{0}$ for our proposed linear approximation method, so we also report the performance of EM-SPLH corresponding to the case $\mathbf{b}=\mathbf{0}$ for completeness. Since there is no interaction term in the formulation of SPLH, all bits learned by finding a fixed point should be the same. Consequently, we evaluate our EM-SPLH and other state-of-the-art baselines by just learning 1-bit code.
For all three datasets, we randomly choose 1000 points as the training set and another 1000 as the test set. The results are the mean of 10 independent partitions.  

\begin{table}[!h]
	\begin{center}
		\begin{tabular}{|c||c|c|c|}
			\hline
			Method & ESPGAME & CIFAR-10 & NUS-WIDE \\
			\hline
			\hline
			LFH	&0.4890&	0.1619&	0.3740
			\\
			\hline
			COSDISH	&0.4773	&0.1054	&0.4065
			\\
			\hline
			NSH	& 0.4957	& \textbf{0.1714}& 0.4646
			\\
			\hline
			BSH	&0.5009	&0.1409	&0.3893
			\\
			\hline
			EM-KSH & {0.4953}&	0.1558&	{0.4703}
			\\
			\hline
			EM-SPLH & \textbf{0.5110}&	0.1573&	\textbf{0.5015}
			\\
			\hline
		\end{tabular}
	\end{center}
	\caption{The mAP performance of learning 1-bit hashing code on three datasets. Best results are in bold.}
	\label{table:3}
\end{table}			
As shown in Table \ref{table:3}, EM-SPLH outperforms several state-of-the-art methods on ESPGAME and NUS-WIDE datasets, including EM-KSH. For CIFAR-10, the performance of our method is also competitive. In fact, for the 1-bit case, formulations of KSH and SPLH are equivalent since $(dS_{ij}-b_ib_j)^2=-2dS_{ij}b_ib_j+const$ holds for $d=1$. 
Besides, LFH, COSDISH, BSH and EM-KSH can all be viewed as optimizing the formulation of KSH (COSDISH and EM-KSH optimize the original KSH, while LFH and BSH optimize the KSH with soft assignments). By comparing the results, we can find the fixed point solved by our method is quite desirable.

\section{Conclusion}
In this paper, we have proposed a novel method to approximate a fixed point of consistency mappings deriving from mean-field inference. We convert these
consistency equations to linear systems by a linear approximation of the sigmoid function to obtain a closed-form solution. By using this technique to supervised hashing problem, we obtain EM-KSH and EM-SPLH. Experimental results on three image datasets show that our methods outperform other state-of-the-art methods.


{\small
	\bibliographystyle{ieee}
	\bibliography{reference}
}

\end{document}